\RequirePackage{etoolbox}

\newcommand{\type}{report}

\newcommand{\verbosity}{short}

\newcommand{\mode}{submission}

\newcommand{\apx}{proofs}

\newcommand{\numberlevel}{}

\newcommand{\counterlevel}{same}

\newcommand{\colorthm}{true}

\newcommand{\usenatbib}{false}

\newcommand{\bibfile}{../library.bib}

\newif\ifsamecounter\ifdefstring{\counterlevel}{same}{\samecountertrue}{\samecounterfalse}
\newif\ifnatbib\ifdefstring{\usenatbib}{true}{\natbibtrue}{\natbibfalse}
\newif\iflong\newif\ifshort\ifdefstring{\verbosity}{long}{\longtrue\shortfalse}{\longfalse\shorttrue}
\ifdefstring{\verbosity}{short}{\longfalse\shorttrue}{\longtrue\shortfalse}
\newif\iffinal\ifdefstring{\mode}{final}{\finaltrue}{\finalfalse}
\newif\ifdraft\ifdefstring{\mode}{draft}{\drafttrue}{\draftfalse}
\newif\ifreview\ifdefstring{\mode}{review}{\reviewtrue}{\reviewfalse}
\newif\ifsubmission\ifdefstring{\mode}{submission}{\submissiontrue}{\submissionfalse}

\newcommand{\natbibstylename}{unsrtnat}
\newcommand{\biblatexstylename}{alphabetic}

\newcommand{\contact}[1]{}
\newcommand{\affiliation}[1]{\\#1}

\ifdefstring{\type}{conference}{%

\documentclass[letterpaper]{article} 
\usepackage{aaai22}  
\usepackage{times}  
\usepackage{helvet}  
\usepackage{courier}  
\usepackage[hyphens]{url}  
\usepackage{graphicx} 
\urlstyle{rm} 
\usepackage{natbib}  
\usepackage{caption} 
\DeclareCaptionStyle{ruled}{labelfont=normalfont,labelsep=colon,strut=off} 
\frenchspacing  
\setlength{\pdfpagewidth}{8.5in}  
\setlength{\pdfpageheight}{11in}  

\setcounter{secnumdepth}{0} 

\renewcommand{\natbibstylename}{aaai22} 

}{\ifdefstring{\type}{journal}{%

\documentclass{report} 

\renewcommand{\natbibstylename}{apalike}

}{

\documentclass[12pt]{xarticle}

\usepackage[a4paper,margin=1in]{geometry}

\usepackage{libertine}

\emergencystretch 1em

}}

\usepackage{amsmath, amsthm, amssymb, thmtools, mathtools}

\usepackage{aligned-overset}

\usepackage[
nottoc
]{tocbibind}

\usepackage{tocloft}
\usepackage{xcolor}

\usepackage{varioref}

\usepackage[
colorlinks = true,
linkcolor = blue!80!black,
urlcolor  = green!50!black,
citecolor = blue!80!black,
anchorcolor = blue!80!black,
]{hyperref}
\usepackage{xstring}

\usepackage{subcaption}

\usepackage{blindtext}

\usepackage{enumitem}

\usepackage[
\ifdraft\else
disable,
\fi
colorinlistoftodos,
prependcaption,
textsize=tiny,
color=red!20,
linecolor=red!20,
shadow,
backgroundcolor=red!25,
bordercolor=red
]{todonotes}
\presetkeys{todonotes}{fancyline}{}

\ifreview
\usepackage{lineno}\linenumbers
\BeforeBeginEnvironment{dmath}{\begin{nolinenumbers}}%
    \AfterEndEnvironment{dmath}{\end{nolinenumbers}}
\fi

\iffinal\else\ifsubmission\else
\usepackage{draftwatermark}
\ifdraft
\SetWatermarkText{Draft Copy - Do Not Distribute}
\SetWatermarkColor[rgb]{0.8,0,0}
\else\ifreview
\SetWatermarkText{Review Copy - Do Not Distribute}
\SetWatermarkColor[rgb]{0,0,0.8}
\fi\fi
\SetWatermarkFontSize{2em}
\SetWatermarkAngle{0.0}
\SetWatermarkHorCenter{\paperwidth/2}
\SetWatermarkVerCenter{2em}
\SetWatermarkScale{1.0}
\fi\fi

\usepackage{tikz}
\usetikzlibrary{automata,positioning,shapes,shapes.geometric,arrows,fit,calc}
\usetikzlibrary{decorations.pathmorphing}

\usepackage{pgfplots}

\usepackage{nameref}

\usepackage{algorithm}
\usepackage[noend]{algpseudocode}
\algnewcommand\Input{\item[{\textbf{Input:}}]}
\algnewcommand\Output{\item[{\textbf{Output:}}]}

\let\OrgBar\|
\usepackage{flexisym}
\usepackage{breqn}
\let\|\OrgBar

\usepackage[nameinlink]{cleveref}

\ifdraft\usepackage{showkeys}\fi%

\ifnatbib
\usepackage{usebib}
\else
\usepackage[
natbib=true,
backend=bibtex,
style=\biblatexstylename,
doi=false,
eprint=false,
isbn=false,
]{biblatex}
\AtEveryBibitem{%
    \clearfield{pages}
    \clearfield{volume}
    \clearfield{number}
    \ifentrytype{online}{}{
        \clearfield{url}
    }
}

\fi

\usepackage[
bibliography=common,
appendix=\ifdefempty{\apx}{inline}{append}
]{apxproof}

\usepackage[many]{tcolorbox}
\tcbset{
breakable,
enhanced,
before skip=\topsep,
after skip=\topsep,
}
\tcbsetforeverylayer{title filled=false}

\usepackage{expl3}

\makeatletter
\newcommand*\ifcounter[1]{%
    \ifcsname c@#1\endcsname
    \expandafter\@firstoftwo
    \else
    \expandafter\@secondoftwo
    \fi
}
\makeatother

\ExplSyntaxOn
\cs_set_eq:NN \Foreach \clist_map_inline:nn
\ExplSyntaxOff

\Foreach{theorem,lemma,proposition,definition,remark,corollary,example,assumption,conjecture}
{
    \ifcounter{#1}{}{
        \ifsamecounter
            \newtheoremrep{#1}[thm]{\MakeUppercase #1}
        \else
            \newtheoremrep{#1}{\MakeUppercase #1}[\numberlevel]
        \fi
    }
    \ifdefstring{\colorthm}{true}{
        \tcolorboxenvironment{#1}{}
    }{}
}

\usepackage{versions}
\ifdefempty{\apx}{
    \excludeversion{proofsketch}
    \excludeversion{inlineproof}
}{}

\usepackage[
]{tcases}


\crefname{subequation}{case}{cases}
\Crefname{subequation}{Case}{Cases}

\usepackage{agi-research}

\usetikzlibrary{arrows,positioning}
\tikzset{
    >=stealth',
    roundbox/.style={
        rectangle,
        rounded corners,
        draw=black, very thick,
        text width=6.5em,
        minimum height=2em,
        text centered},
    thickarrow/.style={
        ->,
        thick,
        shorten <=2pt,
        shorten >=2pt,}
}

\tikzset{
    node distance = 7mm and -3mm,
    innernode/.style = {draw=black, thick, fill=gray!30,
        minimum width=2cm, minimum height=0.5cm,
        align=center},
    outernode/.style = {draw=black, thick, rounded corners, fill=none,
        minimum width=1cm, minimum height=0.5cm,
        align=center, inner sep=0.5cm},
    endpoint/.style={draw,circle,
        fill=gray, inner sep=0pt, minimum width=4pt},
    arrow/.style={->,thick,rounded corners},
    point/.style={circle,inner sep=0pt,minimum size=2pt,fill=black},
    skip loop/.style={to path={-- ++(#1,0) |- (\tikztotarget)}},
    every path/.style = {draw, -latex}
}

\usetikzlibrary{fit,positioning}

\title{
    Reducing Planning Complexity of General Reinforcement Learning with Non-Markovian Abstractions%
    \iflong
    \thanks{A shorter version appeared in the proceedings of the AISTATS 2014 conference \cite{Hutter2016}.}
    \fi
}

\ifdefstring{\type}{conference}{ 

\author {
     Sultan J. Majeed,\textsuperscript{\rm 1}
     Marcus Hutter \textsuperscript{\rm 2}
}
\affiliations {
    \textsuperscript{\rm 1} Research School of Computer Science, ANU\\
    \textsuperscript{\rm 2} Google DeepMind \& Research School of Computer Science, ANU\\
    sultan.pk, hutter1.net
}

}{

\author{
    Sultan J. Majeed
    \contact{sultan.pk} \affiliation{Research School of Computer Science, ANU}
    \and
    Marcus Hutter
    \contact{hutter1.net} \affiliation{Google DeepMind \& Research School of Computer Science, ANU}
}

}

\begin{document}


\ifdraft
    \pagestyle{empty}
    \newpage
    \listoftodos[Todo \& Notes]
    \newpage
    \pagestyle{plain}
    \setcounter{page}{1}
\fi


\maketitle



\begin{abstract}
    The field of General Reinforcement Learning (GRL) formulates the problem of sequential decision-making from ground up. The history of interaction constitutes a ``ground'' state of the system, which never repeats. On the one hand, this generality allows GRL to model almost every domain possible, e.g.\ Bandits, MDPs, POMDPs, PSRs, and history-based environments. On the other hand, in general, the near-optimal policies in GRL are functions of complete history, which hinders not only learning but also planning in GRL. The usual way around for the planning part is that the agent is given a Markovian abstraction of the underlying process. So, it can use any MDP planning algorithm to find a near-optimal policy. The Extreme State Aggregation (ESA) framework has extended this idea to non-Markovian abstractions without compromising on the possibility of planning through a (surrogate) MDP. A distinguishing feature of ESA is that it proves an upper bound of $O\left(\varepsilon^{-A} \cdot (1-\gamma)^{-2A}\right)$ on the number of states required for the surrogate MDP (where $A$ is the number of actions, $\gamma$  is the discount-factor, and $\varepsilon$ is the optimality-gap) which holds \emph{uniformly} for \emph{all} domains. While the possibility of a universal bound is quite remarkable, we show that this bound is very loose. We propose a novel non-MDP abstraction which allows for a much better upper bound of $O\left(\varepsilon^{-1} \cdot (1-\gamma)^{-2} \cdot A \cdot 2^{A}\right)$. Furthermore, we show that this bound can be improved further to $O\left(\varepsilon^{-1} \cdot (1-\gamma)^{-2} \cdot \log^3 A \right)$ by using an action-sequentialization method.
\end{abstract}


\section{Introduction}

Standard Reinforcement Learning (RL), traditionally, models the controlled domain as a finite-state Markov Decision Process (MDP) \cite{Sutton2018}. In an MDP, the most recent observation is a sufficient statistic of the past. However, real-world tasks are more complex and inherently non-Markovian.

\begin{quote}
    ``Markov decision tasks are an ideal. Non-Markov tasks are the norm. They are as
    ubiquitous as uncertainty itself.'' \cite{Lin1992}
\end{quote}

So, if we try to naively model reality (i.e.\ any real-world problem of interest) as an MDP, we will end up with a huge, possibly infinite, number of states, e.g.\ a ``belief'' state MDP \cite{Kaelbling1996}. Note that we will need some type of ``state estimation'' function to extract an MDP state from raw observations \cite{Powell2011}. In this work, we call (a generalized variant of) such state estimation functions \emph{abstraction} maps.

The Partially-Observable MDP (POMDP) \cite{Kaelbling1996} and Predictive State Representation (PSR) \cite{Littman2002} model classes are some of the natural and well-known extensions of MDPs. They allow ``easy'' modeling of non-Markovian nature of an environment. These model classes do provide compact models, but they can have a huge planning complexity, i.e.\ the space and time required to find the optimal policy (the ``best'' future course of actions) \cite{Papadimitriou1987,James2004}. Learning in POMDPs and PSRs is even harder and less understood. 

The problem of General Reinforcement Learning (GRL) starts from the other extreme of making no assumptions about the nature of the domain \cite{Lattimore2013}. Hence, it can comprehend any model class considered in sequential decision-making literature, e.g.\ Bandits (the problems which can be modeled with a single state MDP) \cite{Lattimore2020}, MDPs, POMDPs and PSRs. The GRL framework models the environment as a \emph{history-based} domain, where the future observations can depend on the complete past history of the process \cite{Hutter2000}. In its pure form, the problem of GRL suffers from both $(a)$ the ``state explosion'' like Markovian models as history grows with time and $(b)$ the hard to impossible planning problem like non-Markovian models as no history ever repeats \cite{Hutter2000}.

However, by having no assumptions to start with in GRL opens up the possibility of ``specializing'' the framework by using ``exotic'' abstraction maps (which map any history to a state) including but not limited to POMDPs and PSRs \cite{Hutter2016,Majeed2019}. Importantly, it is possible to show that there exist abstraction maps which can have the best of both Markovian and non-Markovian model classes. These maps can have ``reasonably'' sized (non-Markovian) state-space to model \emph{any} problem and still allow the agent to use MDP style planning (and learning) methods \cite{Hutter2016,Majeed2018,Majeed2019,Majeed2020}. This work is about such powerful mapping functions.

\paradot{Related Work} We are interested in the abstractions of GRL which can guarantee an upper bound on the number of states required to plan (i.e.\ to find a near-optimal policy) in every environment. This problem has first been considered by \citet{Hutter2016} in this very setting. He provided a constructive proof (\Cref{thm:esa-bound}) for an abstraction which has the upper bound on the number of states as
\beq
O\left(\varepsilon^{-A} \cdot (1-\gamma)^{-2A}\right)
\eeq
where $A$ is the number of actions, $\gamma$  is the discount-factor, and $\varepsilon$ is the optimality-gap. We provide more details about these elements later when we formally set up the problem. Recently, \citet{Majeed2020} improved the bound using a general action-sequentialization technique to
\beq
O\left(\varepsilon^{-2} \cdot (1-\gamma)^{-6} \cdot \log^6 A \right)
\eeq

The action-sequentialization method used by \citet{Majeed2020} sequentialize the actions into a binary stream of symbols. This stream of binary input is used to ``mimic'' a ``binarized version'' of the domain.
In this work, we improve both of these results to first
\beq
O\left(\varepsilon^{-1} \cdot (1-\gamma)^{-2} \cdot A \cdot 2^{A}\right)
\eeq
without using action-sequentialization (\Cref{thm:state-bound}) and to
\beq
O\left(\varepsilon^{-1} \cdot (1-\gamma)^{-2} \cdot \log^3 A \right)
\eeq
by using action-sequentialization (\Cref{thm:bin-state-bound}). Typically, state aggregation has been studied as a special case of function approximation \cite{Abel2016,Roy2006,Hutter2019,Hauskrecht2000,Melo2008,Xu2014}. However, our work shows that it might be sufficient for all domains of interest.

\paradot{Notation} \iflong Since the theoretical nature of this paper, it is critical that we use a consistent notation in the paper for the required level of clarity. As it has already been used, \fi $O(f)$ denotes the set of functions of ``order'' $f$. We use $\SetN \coloneqq \{1, 2, 3, \dots\}$ to denote the set of natural numbers starting from $1$, where $\coloneqq$ expresses equality by definition. The set of reals is denoted by $\SetR$. For any arbitrary set $X$, we denote the set of all probability distribution on $X$ as $\Dist(X)$. For any time indexed sequence, we express $x \coloneqq x_n$ and $x' \coloneqq x_{n+1}$ for the current time-step $n$. We use juxtaposition to denote a sequence, string or vector, e.g.\ $x_1 x_2 x_3 \dots \coloneqq (x_1, x_2, x_3, \dots)$. The cardinality of a set (or a sequence) $X$ is expressed by $\abs{X}$. The expectation operator $\E^\pi[Y]$ expresses the expected value of any random variable $Y$ using the underlying (but not explicitly expressed) probability space induced by an action selection policy $\pi$ and environment $\mu$ (which is notationally suppressed). Let $s$ be a label for the set $\{x \in X : \psi(x) = s\}$ then $\psi\inv(s) \in X$ represents an \emph{arbitrary} member from the set, and throughout this work, using this non-standard notation does not lead to ill-defined expressions. We express by $f(x\|yz)$ any real-valued, parameterized function defined over $x \in X$ with parameters $y \in Y$ and $z \in Z$.

\paradot{Paper Structure} \iflong For a better navigation through the paper, we explain here the structure of \else In \fi the rest of the paper. We first formally define the problem setup, including the GRL framework, ESA and the planning problem. At the end of that section, we list the key upper bound on the number of states of the surrogate MDP provided by ESA (\Cref{thm:esa-bound}). Later, we provide the \emph{main results} of this work, which significantly improve the aforementioned upper bound (\Cref{thm:state-bound}). In the same section, we also prove that the bound can be further improved with a (rather ugly\footnote{We use the term ``ugly'' in the sense of it having a comparatively less ``intuitive'' explanation for the binarization method as compared to the ``plain'' abstraction map from histories to states. See \cite{Majeed2020} for more details on action-sequentialization.}) action-sequentialization method (\Cref{thm:bin-state-bound}).
Finally, we conclude the paper by summarizing the contributions and pointing out some key future research directions.

\section{Problem Setup}

In this section, we formally define our problem setup. As described earlier, we use GRL as the foundational framework of this work \cite{Lattimore2013}. This is a relatively detailed section. We believe that to really appreciate the main contributions of this work, one needs a clear understanding about the generality of GRL and the related concepts.

\paradot{General Reinforcement Learning}
Let there be an agent which is choosing its actions from a finite\footnote{The finite set of actions is not much of a restriction. We can almost always convert a real-valued action-space by a sufficiently fine discretization without any adverse effects.\iffalse [REF]? \fi} action-space $\A \coloneqq \{a^1, a^2, \dots, a^A \}$, where the superscript does not denote the exponent but the index of an element. The set of actions can be anything from motor actuations of a robot to any broader notion of ``abstract responses'' of a chat bot. We assume that the environment (i.e.\ our domain of interest) dispenses a precept from a countable percept-space $\OR \coloneqq \{e^1, e^2, e^3, \dots \}$. The percept can be any information received by the agent, e.g.\ a camera image, vital signs of a patient or an input prompt from the user. We do not make the usual MDP (or POMDP) assumption \cite{Sutton2018,Kaelbling1996} on the percept-space, i.e.\ $\OR$ needs not be or have a direct connection with the ``states'' of the environment. \Cref{fig:ae-loop} shows a simple GRL loop. The agent-environment interaction generates a history sequence from the following set of finite histories:
\beq
\H_n \coloneqq (\OR \times \A)^{n-1} \times \OR
\eeq
for any time-step $n \in \SetN$. We denote the set of \emph{all} finite interaction histories as
\beq
\H \coloneqq \bigcup_{m=1}^\infty \H_m
\eeq

In general, the actions of an agent in GRL are decided by a policy $\pi$ which is a complete description for the agent about how to respond at every finite history.
\beq
{\pi : \H \to \Dist(\A)}
\eeq
where $\Dist(\A)$ denotes a probability distribution over the set of actions. It is important to note that $\pi$ could be a \emph{dynamic/non-stationary} policy. The finite history dependence allows us to compactly express any non-stationary policy. Similarly, we say that the environment is ``choosing'' its percepts by the following ``transition'' function:
\beq
{\mu: \H \times \A \to \Dist(\OR)}
\eeq

And, we also assume the existence of an initial percept distribution $\mu \in \Dist(\OR)$, which is dispensed at the start of an agent-environment interaction. As with the policy $\pi$, it is easy to see that $\mu$ encompasses almost every model class considered in the literature \cite{Sutton2018,Kaelbling1996,Littman2002}. The \emph{history-based} environments is arguable the most general class of domains \cite{Hutter2000}. Importantly throughout this work, we \emph{do not} assume any more structure on (the environment) $\mu$.

\begin{figure}[!ht]
    \centering
    \resizebox{0.4\textwidth}{!}{%
        \begin{tikzpicture}
            \node [outernode] (agent) {Agent};
            \node[outernode, left=0.05\textwidth of agent] (env) {Environment};
            \node[endpoint, above= -2pt of env] (or_env) {};
            \node[endpoint, below= -2pt of env] (a_env) {};
            \node[endpoint, below= -2pt of agent] (a_agent) {};
            \node[endpoint, above= -2pt of agent] (or_agent) {};
            \path (a_agent) edge[arrow,bend left] node[below]{$a_n$} (a_env);
            \path (or_env) edge[arrow, bend left] node[above]{$e_n$} (or_agent);
    \end{tikzpicture}}
    \caption{The agent-environment interaction loop.}
    \label{fig:ae-loop}
\end{figure}
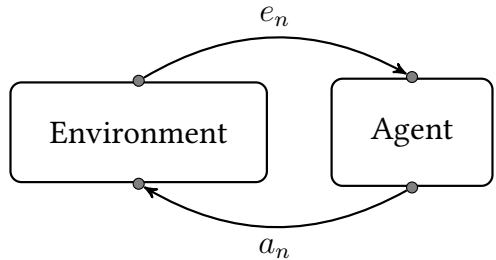

We assume that there exists a reward function which without much loss of generality is bounded.
\beq
{r : \H \to \Delta([r_{\min}, r_{\max}])}
\eeq
where $r_{\min} \leq r_{\max} \in \SetR$, which ``evaluates'' every history. For example, a wining position of a Chess board may be evaluated higher than a lost position. Or, the reward function can be the score achieved by the agent in a computer game \cite{Mnih2015,Silver2016,Silver2018}. As with $\mu$, the reward function $r$ in GRL can also arbitrarily depend on the complete history. For instance, two winning positions can be rewarded differently based on the strategy followed by a Chess playing agent. Without further loss of generality, we assume $r_{\min} = 0$ and $r_{\max} = 1$, as any rescaling of the reward function does not affect the decision-making process \cite{Hutter2000}.

The goal of the agent is to maximize the expected $\g$-discounted sum of rewards, also known as the \emph{history-value} function \cite{Hutter2000}. For any fixed policy $\pi$, we define the history-value function as
\beq
V^\pi(h) \coloneqq \E^\pi\left[\sum_{m=1}^\infty \g^{m-1} r_{\abs{h}+m}\middle| h\right]
\eeq
where $r_n$ is the (possibly random) reward received by the agent at the time-step $n$ and $\g \in [0,1)$ is the discount-factor. The \emph{optimal} history-value function $V^*$ is defined as the supremum over all policies.
\beq
V^*(h) \coloneqq \sup_{\pi} V^\pi(h)
\eeq
for all finite histories $h \in \H$. A policy is $\eps$-optimal (a.k.a. near-optimal) if its history-value function is $\eps$-close to $V^*$ for all histories. That is,
\beq
\sup_{h \in \H} \abs{V^*(h) - V^\pi(h)} \leq \eps
\eeq
and, a policy $\pi^*$ is \emph{optimal} if it is $0$-optimal. For any finite set of actions the optimal history-value function and an optimal policy always exist \cite{Lattimore2014b}. It is easy to see that we can also define a (pseudo\footnote{We call it a pseudo recursion as no history ever repeats.}) recursive form of the history-value function as
\beq\label{eq:be}
Q^\pi(ha) \coloneqq r_\mu(ha) + \g \sum_{e'} \mu(e'\|ha) V^\pi(hae')
\eeq
with a history-action-value function $Q^\pi$ which satisfies the relation
\beq
V^\pi(h) = \sum_{a} \pi(a\|h) Q^\pi(ha)
\eeq
and $r_\mu$ is the expected reward received in the environment $\mu$ after taking action $a$ at history $h$. \Cref{eq:be} is known as Bellman Equation (BE) \cite{Puterman2014}. The same iteration also holds for the optimal value functions \cite{Lattimore2013} as
\beq\label{eq:obe}
Q^*(ha) = r_\mu(ha) + \g \sum_{e'} \mu(e'\|ha) V^*(hae')
\eeq
which is called Optimal Bellman Equation (OBE) \cite{Puterman2014}, where $V^*(h) = \max_a Q^*(ha)$. As mentioned earlier, the goal of an agent in GRL is to behave in a way that the history-value of its policy is close to $V^*$.

\paradot{Abstractions of GRL}
To reduce the countable history-space to a finite state-space, we use deterministic\footnote{Using \emph{stochastic} abstraction maps might be an interesting extension of this setup. However, the deterministic case is sufficient for our purpose (\Cref{thm:state-bound,thm:bin-state-bound}).} abstraction maps. An abstraction $\psi$ is nothing but a map from the set of finite histories $\H$ to a finite\footnote{The finite state-space is \emph{not} a restriction. Later, we show that we only need a finite set of states to achieve near-optimal performance/value in every possible environment.} set of states $\S \coloneqq \{s^1, s^2, \dots, s^S\}$.
\beq
{\psi : \H \to \S}
\eeq

It is appropriate to say that an abstraction is extracting some ``features'' from every history \cite{Hutter2009}. The state (or the feature(s)) is a statistic about the past. For example, the \emph{shopping list} is a feature/state of the \emph{history} for a helper robot who has gone through each kitchen cabinet since morning.

A restricted version\footnote{The state abstraction function in the standard RL literature \cite{Abel2016} can be considered as an abstraction map $\psi: \OR \to \S$.} of the abstraction maps considered in this work is known as a \emph{state aggregation} function in standard RL \cite{Abel2016}.
Standard RL setting \cite{Sutton2018} \emph{is} a special case of a GRL setup with the abstraction $\psi_{\rm MDP}(h) = e_{\abs{h}}(h)$, i.e.\ the recent observation is the current state of the agent. So, standard RL ``works'' on domains where the most recent observation is a \emph{sufficient} statistic of the history \cite{Sutton2018}. Obviously, we cannot ``successfully'' use this abstraction on every domain, \emph{cf}.\ a tracking system which needs at least the last two most recent observations to estimate the velocity of any object being tracked \cite{Kaelbling1996}. We require the abstract policy and value to depend on the state only, but do not require the abstract process to be an MDP (more details later). \Cref{fig:ae-abstraction-loop} shows the information flow of a GRL setup with an abstraction.

The agent can either learn $\psi$ from data\footnote{The \emph{abstraction \cite{Maillard2011} / representation \cite{Bengio2013} / feature  \cite{Hutter2009} learning} problem is orthogonal to this work.} or it has been provided a fixed abstraction, e.g.\ the agent has been ``hard coded'' with some structural knowledge about the world, say the maximum tree depth in MCTS \cite{Silver2018} or a neural network architecture ``decided'' by a human designer \cite{Bengio2013}. It does not matter if $\psi$ is learned or fixed. Once the agent has an abstraction/feature map, the only thing which matters is how ``easily'' the agent can use it to plan a near-optimal policy.

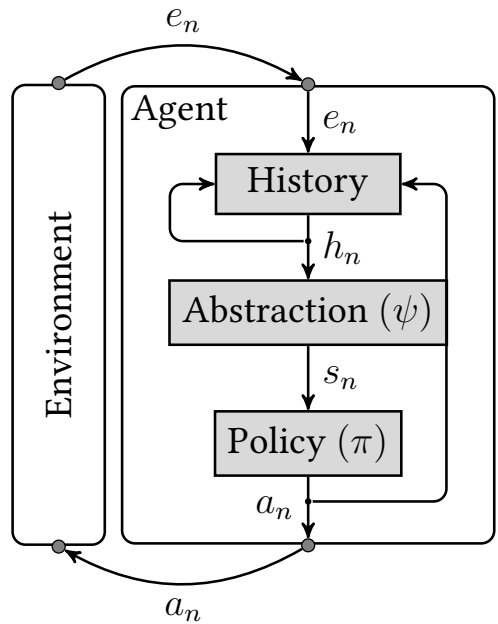
\begin{figure}[!ht]
    \centering
    \resizebox{0.4\textwidth}{!}{%
    \begin{tikzpicture}
        \node (h) [innernode]{History};
        \node (phi) [innernode, below=of h]{Abstraction $(\psi)$};
        \node (pi) [innernode, below=of phi]      {Policy $(\pi)$};
        \node [outernode, align=left, inner sep=0.5cm, fill=none, fit=(h) (phi) (pi), minimum height=5cm] (agent) {};
        \node[below right, inner sep=3pt, fill=none] at (agent.north west) {Agent};
        \node[outernode, left=0.17\textwidth of agent, fit=(agent.north)(agent.south), inner sep=0pt] (env) {};
        \node[below right, inner sep=0pt, fill=none, rotate=90, anchor=center] at (env) {Environment};
        \node[endpoint, above= -2pt of env] (or_env) {};
        \node[endpoint, below= -2pt of env] (a_env) {};
        \node[endpoint, below= -2pt of agent] (a_agent) {};
        \node[endpoint, above= -2pt of agent] (or_agent) {};

        \path (a_agent) edge[arrow,bend left] node[below]{$a_n$} (a_env);
        \path (or_env) edge[arrow, bend left] node[above]{$e_{n}$} (or_agent);
        \path (or_agent) edge[arrow] node[right]{$e_{n}$} (h);
        \path (h) edge[arrow] node[above=0.5pt,midway,name=h_phi,point]{} node[right]{$h_n$} (phi);
        \path (phi) edge[arrow] node[right]{$s_n$} (pi);
        \path (pi) edge[arrow] node[above=0.5pt,midway,name=pi_a,point]{} node[left]{$a_n$} (a_agent);
        \path (pi_a) edge[arrow, skip loop=1.5cm] (h.east);
        \path (h_phi) edge[->, skip loop=-1.5cm, thick, rounded corners] (h.west);
    \end{tikzpicture}}
    \caption{The agent-environment interaction through an abstraction.}
    \label{fig:ae-abstraction-loop}
\end{figure}

\paradot{Planning with Surrogate MDP}
It has already been known that planning over a POMDP or a PSR model is quite hard \cite{Papadimitriou1987,James2004}. The advantage of extra generality of these (standard) models over MDPs is undermined by the complexity of the planning (and even more so learning) part. Even in our setup, the abstraction may lead to a non-MDP state-action sequence $s_1a_1s_2a_2 \dots$, i.e.\ the probability of observing the next state may (still) depend on the complete history. However, we show (\Cref{thm:state-bound,thm:bin-state-bound}) that for some particularly ``useful'' non-Markovian abstractions (\Cref{def:vadp}) we retain the ability of planning using the standard optimized MDP planning/RL algorithms \cite{Bertsekas1996,Sutton2018}.

The agent in our setup ``pretends'' that the state-space of the abstraction is Markovian, even if it is non-Markovian in reality. This pretending induces a (set of) \emph{surrogate} MDP(s) for any abstraction map. The idea of ``pretending'' an abstraction being Markovian may seem naive, but, as we will discuss, it is quite an involved topic \cite{Hutter2016,Majeed2019}. As with the MDP abstraction $\psi_{\rm MDP}$, not every professed \emph{surrogate} MDP of every abstraction leads to a ``good'' policy for the underlying environment, e.g.\ see \citet[Theorem 10]{Hutter2016}. However, among the surrogate MDPs that lead to a near-optimal policy, we want one whose state-space is as \emph{small} as possible to aid planning \cite{Papadimitriou1987}. This is what we show in this work that there exists a class of non-Markovian abstractions (\Cref{def:vadp}) with a much smaller number of states (\Cref{thm:state-bound,thm:bin-state-bound}), as previously known (\Cref{thm:esa-bound}), on which the agent can reliably plan using a surrogate MDP.

But first, let us provide a formal expression of this ``pretend'' surrogate MDP.

\begin{definition}[Surrogate MDP]
For any abstraction $\psi$, the surrogate MDP is defined as

\beq\label{eq:surrogate-mdp}
\b \mu(s'\|sa) \coloneqq \sum_{m=1}^\infty w_m(sa)\sum_{h_m \in \H_m} \mu_\psi(s'\|h_m a)B_m(h_m\|sa)
\eeq
where
\beq
\mu_\psi(s'\|ha) \coloneqq \sum_{e':\psi(hae') = s'} \mu(e'\|ha)
\eeq
is the \emph{induced abstract process} (which is the standard marginal probability distribution of $h' \coloneqq hae'$ mapped to $s'$ given $ha$), $\{w_n(sa): n \geq 1\}$ is a unit summable weighting sequence and $B_n$ is an \emph{action-dependent dispersion distribution} for each time-step $n$ with the ``signature type''
\beq
{B_n: \S \times \A \to \Dist(\H_n)}
\eeq
such that $B_n(h\|sa) \coloneqq 0$ if $\psi(h) \neq s$ and $\abs{h} \neq n$ for any $sa$ state-action pair.

The reward signal is also \emph{abstracted} in the same way as
\beq
\b r_\mu(sa) \coloneqq \sum_{m=1}^\infty w_m(sa)\sum_{h_m \in \H_m} r_\mu(h_ma)B_m(h_m\|sa)
\eeq

Moreover the optimal value functions of the surrogate MDP are
\beq
q^*(sa) \coloneqq \b r_\mu(sa) + \g \sum_{s'} \b\mu(s' \| sa) v^*(s')
\eeq
where $v^*(s) \coloneqq \sum_{a} q^*(sa) \pi^*_{\b\mu}(a\|s)$ and $\pi^*_{\b\mu}$ is an optimal policy of the surrogate MDP.
\end{definition}

It is crucial to note that $B$ is action-dependent which sets it apart from the ``weighting functions'' considered in standard RL literature \cite{Abel2016,Roy2006}. Typically, the action-independent weights are assumed to be normalized stationary distribution over the aggregated underlying states of an MDP. In GRL, the history (i.e. ``underlying'' state of the environment) never repeats. So, there is is not reason for us to assume an action-independent weighting function in our setup.
See \citet{Hutter2016} for a detailed discussion on why \Cref{eq:surrogate-mdp} is indeed the ``correct'' form for the pretended surrogate MDP. Intuitively, \Cref{eq:surrogate-mdp} is a time-averaged, stationary Markovian process ``experienced'' by an agent who is weighing the time by $w$ for each state-action pair. If the analysis (and the guarantees) does not depend on choice of weights then the agent may use any ``model estimation'' method, see \citet{Hutter2016} for an example frequency estimation of the surrogate MDP.

For notational convenience, we sometimes call and use the $w$-weighted distribution $B \coloneqq \sum_m w_mB_m$ simply a \emph{dispersion distribution}.
A point of caution is that $B$ is not an ``ordinary'' probability measure over infinite histories. It is a probability mass function over histories of mixed lengths. \iflong See Appendix for details. \fi It is clear from \Cref{eq:surrogate-mdp} that any ``particular'' choice of $\{B_n, w_n : n \geq 1\}$ implies a ``particular'' surrogate MDP for an abstraction $\psi$. Our target in this work is to be agnostic of this choice which provides a lot of freedom to the agent on the choice of behavior policies \cite{Hutter2016,Majeed2018}.

Now that we have a formal definition of a surrogate MDP $\langle \b \mu, \b r_\mu \rangle$, it begs the question for which type of abstractions it makes sense to use this surrogate MDP to plan? When can we ``uplift'' the optimal policy of the surrogate MDP to a near-optimal policy in the original environment? By uplifting we mean to use the policy
\beq
\u\pi(h) = \pi^*_{\b\mu}(\psi(h))
\eeq
for all $h \in \H$ in the original environment $\mu$, where $\pi_{\b\mu}^*$ is any optimal policy of the surrogate MDP. Moreover, is there any upper bound on the number of states of a surrogate MDP, which is an indicator of planning complexity? \citet{Hutter2016} tackled these questions for the first time in this setting.

\paradot{Extreme State Aggregation} \citet{Hutter2016} started the seminal framework of Extreme State Aggregation (ESA). An ``extreme'' abstraction uses an $\eps$-discretized hypercube $[0, (1-\g)\inv]^A \ni Q^*(h, \cdot)$ for some $\eps > 0$ as its state-space. Each point in the space represents an abstract (quantized) state-action-value ``feature''. He proved the following key result.
\begin{theorem}[{\citealt[Theorem 11]{Hutter2016}}]\label{thm:esa-bound}
    For every environment and $\eps > 0$, there exists an (extreme) abstraction such that any optimal policy of any surrogate MDP is $\frac{\eps}{1-\g}$-optimal\footnote{We use the ``normalized'' $\eps$ to properly represent the optimality-gap scale with respect to the maximum possible value of $1/(1-\g)$. This choice provides a $\g$-independent meaning to $\eps$. For example, $\eps = 10^{-2}$ always implies a $1\%$ error-tolerance (or optimality-gap) for the history-values irrespective from the choice of discount-factor $\g$. The 3 in \citet[Theorem 11]{Hutter2016} can easily be improved to 2 by omitting the first grid-point.} in the original environment. The size of the state-space of the surrogate MDP is uniformly bounded for every environment as
    \beq\label{eq:esa-bound}
    S \leq \fracp{2}{\eps(1-\g)^2}^A
    \eeq
\end{theorem}

The above theorem is quite powerful. It establishes the existence of an abstraction for \emph{every} possible environment with a bounded number of features. One can be sure that a search through a space of models, e.g.\ by doing AutoML/meta learning \cite{Finn2017} or representation learning for a neural networks architecture \cite{Bengio2013}, of this complexity is sufficient for \emph{all} problems of interest. However,
we show that the upper bound in \Cref{eq:esa-bound} (is very loose and) can be significantly further improved.

\section{Improved Upper Bound}

In this section, we provide a novel non-Markovian abstraction map, which is later used to improve the bound in \Cref{eq:esa-bound}. \citet{Hutter2016} did originally consider the following ``coarser'' abstraction map as compared to the novel abstraction map considered in this work (\Cref{def:vadp}).

\begin{definition}[$\eps$-VDP Abstraction]
    An abstraction $\psi$ is $\eps$-VDP if
    \beq
    {\abs{V^*(h) - V^*(\d h)} \leq \eps} \land {\pi^*(h) = \pi^*(\d h)}
    \eeq
    for any pair of histories $h$ and $\d h$ mapped to the same state, i.e.\ $\psi(h) = \psi(\d h)$.
\end{definition}

However, Jan Leike soon refuted this abstraction by providing a comprehensive counter-example for such maps \cite[Theorem 10]{Hutter2016}. The counter-example suggests that $\eps$-VDP abstraction abstracts ``too much''. It is no longer able to represent the near-optimal policies of the original environment as optimal policies of the surrogate MDP(s). \iflong It is important to stress that we are interested in the abstractions which allow the use of surrogate MDPs. Otherwise it is trivially easy to represent any optimal policy of the environment with $A$ states, i.e.\ simply lump the histories together with same optimal actions:
\beq
{\psi(h) = \psi(\d h) \iff \pi^*(h) = \pi^*(\d h)}
\eeq
\fi

The original motivation of considering $\eps$-VDP abstraction by \citet{Hutter2016} was to improve the state bound of \Cref{thm:esa-bound}. \citet[Theorem 11]{Hutter2016} is based on the abstractions which lump $\eps$-close optimal history-action-value functions $Q^*$. One of the main reasons of choosing $Q^*$ was to easily represent the optimal policy $\pi^*$ which is simply the $\argmax$ of $Q^*$. However, the state-space bound turns out to be exponential in $A$. However, if he had a similar result for $\eps$-VDP abstractions then the bound would have turned out to be linear in $A$. Surprisingly, \citet{Hutter2016} was able to show that (extreme) abstractions based on $Q^\pi, Q^*, V^\pi$ can be used to get corresponding upper bounds on the state-space size but it is not possible with $V^*$. In this work, we show that we need to ``refine'' $\eps$-VDP abstractions a bit more to regain the ability of using a surrogate MDP for planning. We define this novel non-Markovian abstraction as follows.

\begin{definition}[$\eps$-VADP Abstraction]\label{def:vadp}
    An abstraction $\psi$ is $\eps$-VADP if
    \beq
    {\abs{V^*(h) - V^*(\d h)} \leq \eps} \land {\pi^*(h) = \pi^*(\d h)}
    \eeq
    and,
    \beq
    {\A_{\eps'}(h) = \A_{\eps'}(\d h)}
    \eeq
    for any pair of histories $h$ and $\d h$ mapped to the same state, i.e.\ $\psi(h) = \psi(\d h)$, and
    \beq
    \A_{\eps'}(h) \coloneqq {\{a \in \A \mid V^*(h) - Q^*(ha) \leq \eps' \}}
    \eeq
    is the set of \emph{all} $\eps'$-optimal actions and $\eps' := \fracp{1+ 3\g}{1-\g}\eps$
\end{definition}

Intuitively. an $\eps$-VADP abstraction produces features/states from histories which not only have the (approximate) equal history-values and same optimal action but it also ``identifies'' an $\eps'$-wide ``buffer zone''. So, later, when the surrogate MDP may pick a ``wrong'' action (due to the approximate aggregation), there should be enough ``error margin'' between the near-optimal actions, $\A_{\eps'}$, and the rest of the non-optimal actions, $\A \setminus \A_{\eps'}$. So, any uplifted optimal policy of the surrogate MDP will still be $\eps'$-optimal in the original environment. The counter-example to $\eps$-VDP in \citet{Hutter2016} relied on the lack of this ``buffer zone'' between the optimal and sub-optimal actions. In the rest of the section, we formally prove this fact. Before that, we need the ``$B$-averaged'' optimal history-action-value function.
\beq
\b Q^*(sa) \coloneqq \sum_{m=1}^\infty w_m(sa)\sum_{h_m \in \H_m} Q^*(h_ma)B_m(h_m\|sa)
\eeq

Luckily, we can show that there exists a relationship between the optimal state-action-value function with this ``baseline'' $B$-averaged state-action-value function (\Cref{lem:q-b-q}). This allows us to directly ``see'' and use the error margins between the aggregated history-action values and state-action values. We need the following stepping stone lemmas (\Cref{lem:abc,lem:max-q-v-rep}) to prove this key relationship.

\begin{lemma}[{$\max-\min$} compute]\label{lem:abc}
    If $\pi^*(h) = \pi^*(\d h)$ for any pair of histories $h$ and $\d h$ mapped to the same state, i.e.\ $\psi(h) = \psi(\d h)$, then the following holds:
    \beq
    \max_a \min_{h:\psi(h) = s} Q^*(ha) = \min_{h:\psi(h) = s} \max_a Q^*(ha)
    \eeq
\end{lemma}
\begin{proof}
    The upper bound trivially holds because of the minimax theorem.
    \beq
    \max_a \min_{h:\psi(h) = s} Q^*(ha) \leq  \min_{h:\psi(h) = s} \max_a Q^*(ha)
    \eeq
    The lower bound can also be proven as follows:
    \beq
    \max_a \min_{h:\psi(h) = s} Q^*(ha)
    \overset{(a)}{\geq} \min_{h:\psi(h) = s} Q^*(h\pi^*(\psi\inv(s)))
    \overset{(b)}{\geq} \min_{h:\psi(h) = s} \max_{a} Q^*(ha)
    \eeq
    where $(a)$ is true for any fixed action, and we have chosen the ``preserved'' optimal action $\pi^*$ and $(b)$ is true by the assumption of the lemma.
\end{proof}
\begin{lemma}[{$\max$ relationship}]\label{lem:max-q-v-rep}
    For any $B$ and $\eps$-VDP abstraction, the following holds:
    \beq
    \abs{\max_a \b Q^*(sa) - V^*(\psi\inv(s))} \leq \eps
    \eeq
    for every state $s$, where $V^*(\psi\inv(s))$ is the history-value of any \emph{representative} history mapped to the state.
\end{lemma}
\begin{proof}
    We start by showing the upper bound.
    \beq
    \max_a \b Q^*(sa)
    \overset{(a)}{\leq} {\max_a \max_{h} Q^*(ha) = \max_h \max_a Q^*(ha)}
    = {\max_h V^*(h) \overset{(b)}{\leq} V^*(\psi\inv(s)) + \eps}
    \eeq
    where $(a)$ is true due to the fact that expectation is upper bounded by the maximum value and $(b)$ is because of $\eps$-value uniformity of the abstraction. Now we prove the lower bound.
    \beq
    \max_a \b Q^*(sa)
    \overset{(d)}{\geq} \max_a \min_{h:\psi(h) = s} Q^*(ha)
    \overset{(e)}{=} {\min_{h:\psi(h) = s} \max_{a} Q^*(ha) = \min_{h:\psi(h) = s} V^*(h)}
    \overset{(f)}{\geq} V^*(\psi\inv(s)) - \eps
    \eeq
    where $(d)$ holds by the fact that minimum value lower bounds the expectation, $(e)$ is due to \Cref{lem:abc}, and $(f)$ us using $\eps$-uniformity of the optimal value.
\end{proof}

\begin{lemma}\label{lem:q-b-q}
    For any surrogate MDP of an $\eps$-VDP abstraction, the following holds:
    \beq
    \abs{q^*(sa) - \b Q^*(sa)} \leq \frac{2\g\eps}{1-\g}
    \eeq
    for any $sa$-pair.
\end{lemma}
\begin{proof}
    Let $\delta := \sup_{sa} \abs{q^*(sa) - \b Q^*(sa)}$. First, we need the following bound to prove the main result.
    \beq\label{eq:v-v-rep}
    \abs{v^*(s) - V^*(\psi\inv(s))}
    = \abs{\max_a q^*(sa)  - \max_a \b Q^*(sa) + \max_a \b Q^*(sa) - V^*(\psi\inv(s)) }
    \overset{(a)}{\leq} {\max_a \abs{q^*(sa)  - \b Q^*(sa)} + \abs{\max_a \b Q^*(sa) - V^*(\psi\inv(s))}}
    \overset{(b)}{\leq} \delta + \eps
    \eeq
    where $(a)$ is a simply the triangular inequality and the mathematical fact that $\abs{\max_x g(x) - \max_x f(x)} \leq \max_{x}\abs{g(x) - f(x)}$, and $(b)$ uses the definition of $\delta$ and \Cref{lem:max-q-v-rep}.
    \beq
    \abs{q^*(sa) - \b Q^*(sa)}
    \overset{(c)}{=}\g \abs{\sum_{s'} \b\mu(s'\|sa)v^*(s') - \sum_{h}B(h\|sa)\sum_{e'}\mu(e'\|ha)V^*(hae')}
    \overset{(d)}{\leq} \g \sum_{s'} \b\mu(s'\|sa)\abs{v^*(s') - V^*(\psi\inv(s'))} + \g \sum_{h}B(h\|sa) \times \sum_{e':\psi(hae') = s'} \mu(e'\|ha) \abs{V^*(\psi\inv(s')) - V^*(hae')}
    \overset{(e)}{\leq} {\g \sum_{s'} \b\mu(s'\|sa)\abs{v^*(s') - V^*(\psi\inv(s'))} + \g \eps \overset{(f)}{\leq} \g \delta + 2\g \eps}
    \eeq
    where $(c)$ is using the definitions of $q^*$ and $\b Q^*$, $(d)$ is simple algebra, $(e)$ uses $\eps$-uniformity of the optimal value, and $(f)$ is due to \Cref{eq:v-v-rep}. Taking $\sup_{sa}$ on the l.h.s. and solving w.r.t. $\delta$, we have the claim as $\delta \leq \frac{2\g\eps}{1-\g}$.
\end{proof}

The above lemma is already hinting about the affects of the approximation error (or optimality-gap) $\eps$ on the choice of the optimal actions of the surrogate MDP. Because  $q^*$ and $\b Q^*$ are not exactly the same, therefore with any small discrepancy, the agent can be ``fooled'' into choosing a non-optimal action. The surrogate MDP can wrongly favor a ``non-optimal'' action as $\pi^*_{\b\mu}$ instead of $\pi^*$ because of this error margin, which could be an arbitrarily worse choice in some histories. We show in the following lemma that $\eps$-VADP abstractions do not allow for such errors. The surrogate MDP can only ``wrongly favor'' actions from $\A_{\eps'}$.

\begin{lemma}\label{lem:no-sub-support}
    Any optimal policy of any surrogate MDP of an $\eps$-VADP abstraction has support \emph{only} in $\A_{\eps'}(s)$ for every state $s$, where $\eps' := \fracp{1+3\g}{1-\g}\eps$.
\end{lemma}
\begin{proof}
    Let $b$ be any action in $\A\setminus\A_{\eps'}(s)$ for any state $s$. We show that $b$ is not an optimal action on $s$ in any surrogate MDP, i.e.\ $v^*(s) > q^*(sb)$.
    \beq
    v^*(s) = {\max_{a} q^*(sa) \overset{(a)}{\geq} \max_{a} \b Q^*(sa) - \frac{2\g\eps}{1-\g}}
    \overset{(b)}{\geq} \max_{a} \min_{h\in\psi\inv(s)} Q^*(ha) - \frac{2\g\eps}{1-\g}
    \overset{(c)}{=} \min_{h\in\psi\inv(s)} V^*(h) - \frac{2\g\eps}{1-\g}
    \overset{(d)}{\geq}  \max_{h\in\psi\inv(s)} V^*(h) - \eps - \frac{2\g\eps}{1-\g}
    \overset{(e)}{>}  \max_{h\in\psi\inv(s)} Q^*(hb) + \frac{(1 + 3\g) \eps}{1-\g} - \frac{(1 + \g) \eps}{1-\g}
    \overset{(f)}{\geq} {\b Q^*(sb) + \frac{2\g\eps}{1-\g} \overset{(g)}{\geq} q^*(sb)}
    \eeq
    where $(a)$ and $(g)$ both use \Cref{lem:q-b-q}, $(b)$ and $(f)$ are due to the fact that $\max X \geq \E[X] \geq \min X$, $(c)$ is true by \Cref{lem:abc}, $(d)$ holds due to $\eps$-uniformity of the optimal value, and (the most important step) $(e)$ is true by the sub-$\eps'$-optimality of action $b$, which is guaranteed by the abstraction.
\end{proof}

The above lemma has already done the heavy lifting for us. Now, we can easily prove that the uplifted policy of any surrogate MDP is near-optimal in the original environment.

\begin{theorem}\label{thm:policy-uplift}
    Any optimal policy of any surrogate MDP of an $\eps$-VADP abstraction is $\fracp{1+3\g}{(1-\g)^2}\eps$-optimal in the original environment.
\end{theorem}
\begin{proof}
    \Cref{lem:no-sub-support} establishes that any optimal policy $\pi^*_{\b\mu}$ of any surrogate MDP $\b\mu$ has support only in $\A_{\eps'}$, i.e.\ $\pi^*_{\b\mu}(s) \in \A_{\eps'}(s)$ for all $s$, where $\eps' := (\frac{1+3\g}{1-\g})\eps$. Hence, if we uplift this policy to the original environment as $\u\pi(h) := \pi^*_{\b\mu}(\psi(h))$ for every history $h$, then we have the guarantee that $V^*(h) - Q^*(h\u\pi(h)) \leq \eps'$.
    Therefore by \citet[Theorem 7]{Hutter2016}, we get the claim that $0 \leq V^*(h) - V^{\u\pi}(h) \leq \frac{\eps'}{1-\g}$ for every history $h$.
\end{proof}

Building on the above result, we provide the main claim of this work.

\begin{theorem}\label{thm:state-bound}
    For every environment and $\eps > 0$, there exists an (extreme) abstraction such that any optimal policy of any surrogate MDP is $\frac{\eps}{1-\g}$-optimal in the original environment. The size of the state-space of the surrogate MDP is uniformly bounded for every environment as
    \beq
    S \leq {\frac{(1+3\g)A\cdot 2^{A-1}}{\eps(1-\g)^2} \leq \frac{4A\cdot 2^{A-1}}{\eps(1-\g)^2}}
    \eeq
\end{theorem}
\begin{proof}
    We provide a constructive proof of the claim. Let $\psi_\mu$ be an (extreme) $\fracp{1-\g}{1+3\g}\eps$-VADP abstraction of the environment as
    \beq
    \psi_\mu(h) \coloneqq \left(\left\lceil\frac{V^*(h)(1+3\g)}{\eps(1-\g)}\right\rceil, \pi^*(h), \A_{\eps}(h) \right)
    \eeq
    for every history $h$. The above equation is basically a ``feature'' extractor. It is putting the history-value its own discretized bin, identifying the optimal action, and finally labeling all $\eps$-optimal actions of the history, which is exactly the definition of an $\frac{\eps}{1-\g}$-VADP abstraction (\Cref{def:vadp}).
    Now, by \Cref{thm:policy-uplift} we are guaranteed that the uplifted policy of any surrogate MDP of $\psi_\mu$ is $\frac{\eps}{1-\g}$-optimal in the original environment. Moreover, it is easy to see that the size of the state-space is
    \beq\label{eq:max-states}
    S = \left\lfloor\frac{1+3\g}{\eps(1-\g)^2}\right\rfloor \times A \times 2^{A-1}
    \eeq
    where the first term is the number of bins a value can be put in, the second term is the number of different optimal actions, and the third value is the number of different remaining $A-1$ actions which can be $\eps$-optimal. This proves the main claim.
\end{proof}

Interestingly, there is a way to further improve this bound. However for that to happen, we need another element in our setup. \citet{Majeed2020} introduced a general method of sequentializing the decision-making process for any environment. We can do \emph{binarization} of the environment, as described by \citet{Majeed2020}, by (logically) putting a ``binary-mock'' around the environment which only takes binary input as actions, i.e.\ $\A = \{a^1, a^2\}$ and respond to the agent with a previously buffered percept and reward. The agent is ``effectively'' interacting with a ``binarized version'' of the environment where it only takes binary decisions, see \citet{Majeed2020} for more details about the action-sequentialization setup. Besides a significantly reduced state-space, also the action-space reduces in the surrogate MDP to 2 actions.

\begin{theorem}\label{thm:bin-state-bound}
    For every environment and $\eps > 0$, there exists an (extreme) abstraction of the \emph{binarized version} of the environment such that any optimal policy of any surrogate MDP is $\frac{\eps}{1-\g}$-optimal in the original environment. The size of the state-space of the surrogate MDP is uniformly bounded for every environment as
    \beq\label{eq:first}
    S \leq \frac{16(\ceil{\log_2 A} + 1 - \g)^3}{\g \eps (1-\g)^2}
    \eeq
    which simplifies to
    \beq\label{eq:second}
    S \leq 17\ceil{\log_2 A}^3 \eps\inv (1-\g)^{-2}
    \eeq
    when $\g \geq 0.99$ for any $A$ sized action-space.For large action-spaces, $\g \geq 0.95$ is sufficient for the simplified bound.
\end{theorem}
\begin{proof}
    Let $\lambda \coloneqq \g^{1/d}$ be the discount-factor in the binarized environment, where each action is represented by $d = \ceil{\log_2 A}$ bits. We start by noting that any $\frac{\lambda^{d-1}\eps}{1-\lambda^d}$-optimal policy in the binarized version of the environment is $\frac{\eps}{1-\g}$-optimal in the original environment \cite[Theorem 4.6]{Majeed2020}. So, we need to find $\frac{\lambda^{d-1}\eps}{1-\lambda^d}$-optimal policy of the binarized environment through a surrogate MDP by using an abstraction. By \Cref{thm:policy-uplift}, we know that an $\eps'$-VADP abstraction leads to a $\frac{(1+3\lambda)\eps'}{(1-\lambda)^2}$-optimal policy in the binarized environment. Hence by equating $\frac{\lambda^{d-1}\eps}{1-\lambda^d}$ with $\frac{(1+3\lambda)\eps'}{(1-\lambda)^2}$, we get
    \beq
    \eps' = \frac{\lambda^{d-1}(1-\lambda)^2\eps}{(1-\lambda^d)(1+3\lambda)}
    \eeq
    which implies an $\frac{\eps}{1-\g}$-optimal policy in the original environment. Therefore,
    \beq
    S \overset{(a)}{=} \left\lfloor\frac{(1-\g)(1+3\lambda)}{\lambda^{d-1}(1-\lambda)^3\eps}\right\rfloor \times 2 \times 2^{2-1}
    \overset{(b)}{\leq} \frac{4(1-\g)(d + 1 -\g)^3(1+3\lambda)}{\lambda^{d-1}(1-\g)^3\eps}
    \overset{(c)}{\leq} {\frac{4(d + 1 -\g)^3(1+3\lambda)}{(1-\g)^2\g\eps} \overset{(d)}{\leq} \frac{16(d + 1 -\g)^3}{(1-\g)^2\g\eps}}
    \eeq
    where $(a)$ is \Cref{eq:max-states} with $\eps$ replaced by $\eps'$, $(b)$ follows the same steps as in \citet[Equation (23)]{Majeed2020}, $(c)$ is due to $\lambda^{d-1}\eps \geq \g\eps$, and $(d)$ holds because $\g, \lambda \leq 1$, which is the bound in \Cref{eq:first}. Now, we further simplify the bound for large $\g$.
    Let $\delta \coloneqq 1-\g \leq 0.3$ then
    \beq
    \frac{16(d + 1 -\g)^3}{\g}
    = {\frac{16(d + \delta)^3}{1-\delta} = 16d^3\fracp{(1 + \delta/d)^3}{1-\delta}}
    \overset{(e)}{\leq} {16d^3\fracp{1 + 4\delta}{1-\delta} \overset{(f)}{\leq} 17d^3}
    \eeq
    where $(e)$ holds because $d \geq 1$ and $\delta \leq 0.3$, and $(f)$ follows from simple algebra when $\delta \leq \frac{1}{81}$ which implies $\g \geq 0.99$.

\end{proof}

This concludes the main contributions of this work.

\section{Conclusions}

In this work, we considered the problem of modeling any history-based problem by a fixed-sized non-Markovian abstraction while still retaining the ability to plan using a surrogate MDP. We proved that there exist non-Markovian abstractions (\Cref{def:vadp}) which can model \emph{any} problem in $O(\eps^{-1} \cdot (1-\g)^{-2} \cdot A \cdot 2^A)$ number of states (\Cref{thm:state-bound}), which is a significant improvement over the previously known upper bound (\Cref{thm:esa-bound}). Furthermore, we get an even tighter upper bound of $O(\eps^{-1} \cdot (1-\g)^{-2} \cdot \log^3 A)$ (\Cref{thm:bin-state-bound}) if we do action-sequentialization.

Our results can help guide the abstraction/representation learning methods \cite{Bengio2013,Hutter2009,Maillard2011}. This work provides the sufficient conditions on the size and nature of the state-space. Therefore, any architectural search algorithm, e.g. representation/meta learning \cite{Finn2017,Bengio2013}, can focus the search in a VADP-like model class with a maximum number of states proved in this work. Moreover, our work helps weaken the typical assumption of the true model being in the class \cite{Maillard2011,Hutter2009} to a much smaller class of models, \emph{cf.} the assumption that the true model is in an MDP class requires a huge model class before it even starts to approximate any history-based process \cite{Powell2011}.

The bound in \Cref{thm:bin-state-bound} might be the tightest bound possible without loosing the ability to use a surrogate MDP for planning. Providing an ideally matching lower bound is a key future direction of this work. Moreover, it is interesting to see if the bound of \Cref{thm:bin-state-bound} is achievable without action-sequentialization.

\iffinal
\section{Acknowledgments}

\iflong We thank \emph{Joe Smith} and \emph{John Smith} for
proofreading earlier drafts and the anonymous reviewers for their valuable feedbacks.\fi \ This work has been supported by Australian Research Council grant DP150104590.
\fi


\begin{small}
\ifnatbib\bibliography{\bibfile}\else\printbibliography\fi
\end{small}

\end{document}